\documentclass[12pt]{article}
\usepackage{microtype}
\usepackage{graphicx}
\usepackage{subfigure}
\usepackage{amsmath}
\usepackage{bbm}
\usepackage{amssymb}
\usepackage{mathtools}
\usepackage[shortlabels]{enumitem}
\usepackage{booktabs} %
\usepackage{amsmath,amssymb,amsfonts}
\usepackage{setspace}
\usepackage{graphicx}
\renewcommand\labelenumi{(\roman{enumi})}
\renewcommand\theenumi\labelenumi
\usepackage{amsthm}

\usepackage{bbm}

\usepackage{textcomp}
\usepackage{xcolor}
\usepackage{amsmath}
\usepackage{bbm}
\usepackage{multicol}

\def\lf{\left\lfloor}   
\def\rf{\right\rfloor}

\newcounter{relctr} %
\everydisplay\expandafter{\the\everydisplay\setcounter{relctr}{0}} %

\AtBeginDocument{} %

\newtheorem{theorem}{Theorem}%

\newtheorem{corollary}{Corollary}[theorem]
\newtheorem{proposition}[theorem]{Proposition}

\usepackage{natbib}

\newcommand{\beq}{\begin{eqnarray*}}
\newcommand{\eeq}{\end{eqnarray*}}
\newcommand{\beqn}{\begin{eqnarray}}
\newcommand{\eeqn}{\end{eqnarray}}

\newcommand{\hide}[1]{}

\usepackage{fancyvrb} %

\renewcommand{\P}{\mathbb{P}}

\DeclareMathOperator*{\argmin}{arg\,min}
\DeclareMathOperator*{\argmax}{arg\,max}

\newcommand{\nrm}[1]{\left\Vert #1 \right\Vert}

\newcommand{\vertiii}[1]{{\left\vert\kern-0.25ex\left\vert\kern-0.25ex\left\vert #1 
    \right\vert\kern-0.25ex\right\vert\kern-0.25ex\right\vert}}

\usepackage{hyperref}

\newcommand{\blind}{1}

\begin{document}

\def\spacingset#1{\renewcommand{\baselinestretch}%
{#1}\small\normalsize} \spacingset{1}

\if1\blind
{
  \title{\bf Near Optimal Inference for the Best-Performing Algorithm}
  \author{    Amichai Painsky\\
    Tel Aviv University, Israel}
 \date{} 
  \maketitle
} \fi

\bigskip

\begin{abstract}
Consider a collection of competing machine learning algorithms. Given their performance on a benchmark of datasets, we would like to identify the best performing algorithm. Specifically, which algorithm is most likely to rank highest on a  future, unseen dataset.  A natural approach is to select the algorithm that demonstrates the best performance on the benchmark. However, in many cases the performance differences are marginal and additional candidates may also be considered. This problem is formulated as subset selection for multinomial distributions. Formally, given a sample from a countable alphabet, our goal is to identify a minimal subset of symbols that includes the most frequent symbol in the population with high confidence. In this work, we introduce a novel framework for the subset selection problem. We provide both asymptotic and finite-sample schemes that significantly improve upon currently known methods. In addition, we provide matching lower bounds, demonstrating the favorable performance of our proposed schemes. 
\end{abstract}

\noindent%
\vfill

\newpage
\spacingset{1.5} %

\section{Introduction}\label{Introduction}
In a standard comparative study, multiple machine learning algorithms are evaluated across a variety of datasets. Typically, the goal is to identify the \textit{best-preforming} algorithm. Unfortunately, the results are rarely conclusive and strongly depend on the performance criterion. For example, consider a study focusing on classification algorithms. The best-performing algorithm may be defined as the one that achieves the best results (in terms of Accuracy) on the average. Alternatively, we may be interested in the algorithm that is most likely to attain the best Accuracy over a future (unseen) dataset. In addition, we are typically interested in statistically significant conclusions. For example, is the top performer on the benchmark truly the best? Are there additional contenders? How many datasets are needed to confidently identify a single best-performing algorithm?

In this work, we address the fundamental problem of identifying the best-performing algorithm through statistical inference. Our objective is to construct the smallest possible subset of algorithms that, with high confidence, includes the top-performing algorithm. We argue that a natural performance measure is the probability of winning a future dataset. That is, we would like to identify the algorithm that is most likely to attain the best result (for example, highest Accuracy) on a future dataset. In practice, for every competing algorithm we first estimate the probability that it wins (attains the highest Accuracy) over an unseen dataset. Based on these estimates, we  identify a minimal subset that is likely to contain the most probable winner (hence, best-performing algorithm). While our objective is not new in the machine learning literature, existing statistical tests and post-hoc inference methods often lack sufficient power. We review and analyze these prior approaches in Section \ref{related work}.

Our main contribution is two-fold. First, we reformulate the proposed inference task as subset selection for multinomial distribution -- a problem that has been widely studied, with several notable contributions. Second, we introduce a novel solution to this problem that significantly improves upon currently known methods. We distinguish between asymptotic and finite sample regimes, and establish the near-optimality of proposed methods by deriving a matching lower bound.  Finally, we apply the resulting inference schemes to synthetic and real-world comparative studies, demonstrating the favorable results they achieve, compared to existing alternatives.

It is important to emphasize the practical implications of our work. In comparative studies, a collection of ML algorithms is applied to a variety of datasets. Naturally, data-handling, pre-processing and hyperparameter tuning have a significant role on reported results. In this work, we implicitly assume that the performance of each algorithm reflects the end-to-end pipeline, according to recommended guidelines. In that sense, these results represent the outcome a practitioner could realistically achieve. Our contribution lies in proposing a novel post-processing inference scheme, given these results.

\section{Related Work}\label{related work}

Given the results of multiple algorithms over a collection of datasets, we would like to identify the best-performing algorithm. There exist many approaches for this task which differ by their objective, statistical orientation, the form of exposition (qualitative or quantitative), visualization considerations and so forth. Here, we briefly review some of the more popular merits and discuss their advantages and caveats with respect to our goal.  

The \textit{average performance} is a simple averaging of the algorithms' results over the benchmark datasets, according to a predetermined figure of merit. For example, the averaged mean square error (MSE) for regression problems. While this measure is simple and intuitive, it fails to capture the differences between the datasets. For example, consider two datasets where the typical MSE in one dataset is of scale of $10^{-1}$, while in the other is of scale of $10^2$. It is quite misleading to directly average them, since the first dataset would be negligible. To overcome this difficulty, the \textit{average ranking} considers the average of the rankings that every algorithm achieves. The top performing algorithm is ranked first, followed by the runner-up (ranked second) and so forth. Hence, a lower averaged ranking implies better performance across all datasets. The average ranking is perhaps the most commonly used criterion for comparing multiple algorithms simultaneously \citep{demvsar2006statistical,fernandez2014we}. It is very intuitive and explainable. However, it does have some major caveats. For example, it implicitly assumes that all the rankings should be equally weighted. In other words, the difference between being ranked first and fifth is the same as between $21^{th}$ and $25^{th}$. Obviously, this assumption is problematic. While the former  ($1^{th}$ vs. $5^{th}$) suggests dominance, the latter is typically random, especially in the presence of many competing algorithms. Recently, \cite{fernandez2014we} proposed an additional measure of dominance. The Probability of Achieving the Best Accuracy (PAMA) criterion measures the percentage of datasets on which a given algorithm achieves the best result.  This measure is of high interest to most predictive modeling tasks, where the goal is to determine which algorithm to apply on field. In fact, it corresponds to a maximum likelihood estimator (MLE) for the probability of winning a future dataset. Notice that this criterion  implicitly assumes that the win probability only depends on past wins, and not the complete ranking. Yet, it directly considers our objective of interest. 

Once we establish our performance criterion, we are typically interested in a corresponding statistical test. A common example is pair-wise comparisons. For every pair of algorithms we test the null hypothesis that they perform equally well. We repeat this process for every pair of algorithms and correct for multiplicity \citep{demvsar2006statistical}. However, this routine is not always adequate. First, the number of hypotheses grows quadratically with the number of algorithms, which makes the multiplicity correction very conservative. Second, the result may not be clear enough. That is, a typical outcome may be ``Algorithm $A$ is better than $C$, and $B$ is better than $E$", as the rest of comparisons do not yield statistically significant conclusions. Naturally, there exist more sophisticated statistical tests for the problem. The ANOVA is a parametric test \citep{fisher1956statistical} which defines a null hypothesis that all the algorithms perform equally well. Notice that this approach does not suggest which is the better algorithm.  The Friedman $\chi^2$-test \citep{friedman1937use,friedman1940comparison}, and the enhanced $\text{F}$-test  \citep{iman1980approximations}, introduce simple and non-parametric statistics for the same task, which depend on the average ranks of the studied algorithms. These tests generalize the parametric ANOVA, which assumes an underlying normal distribution. If the null hypothesis is rejected, then we can proceed with a post-hoc test. The Nemenyi test \citep{nemenyi1963distribution} is similar to the Tukey test for ANOVA  \citep{tukey1949comparing} and is used when all the algorithms are compared to each other. Specifically, the performance of two algorithms is significantly different if their corresponding average ranks $R_u$ and $R_u$ differ by at least the critical difference $$\text{CD}=q_\delta\sqrt{\frac{A(A+1)}{6n}}$$ 
where $A$ is the number of algorithms, $n$ is the number of evaluated datasets and the critical values $q_\delta$ are based on the Studentized range statistic divided by $\sqrt{2}$ (see Table $5(a)$ of \cite{demvsar2006statistical}). 
Yet, it is important to emphasize that this test is typically not powerful enough due to its implicit correction for multiple comparison (as it grows with $A$). Specifically, as \cite{demvsar2006statistical} indicates, the Friedman test may report a significant difference but the post-hoc test fails to detect it. In addition, \cite{demvsar2006statistical} discusses the case where the algorithms are compared to a predefined control algorithm, or if multiple comparisons that control the \textit{family-wise error rate} (FWER) or the \textit{false discovery rate} (FDR) are of interest. These tests are less relevant to our problem of interest, as later discussed in Section \ref{experiments}. 

An additional related inference problem is \textit{rank verification}. Given the ordered performance (in any measure criterion), we would like to validate that the order is statistically significant. For example, assume that we are interested in the probability of winning future datasets, following PAMA \citep{fernandez2014we}. We count the number of wins per each algorithms, and sort the results accordingly. Given the sorted performance, can we conclude that the best-performing algorithm is really the one with the highest count? Is the runner-up really the second best? Are the ranks statistically significant? \cite{hung2019rank} answer these questions with a simple top-down procedure: they begin with an unadjusted pairwise test comparing the winner to the runner-up. If the test rejects at a predefined confidence level, they reject the null and declare that the winner is really the best. They continue by comparing the runner-up to the second runner-up, again using the unadjusted pairwise test, and so on. Interestingly, this procedure was shown to be exact. Additional notable contributions to the rank verification problem appear, for example, in \cite{gutmann1987selected,bofinger11991selecting,karnnan2009does,maymin1992testing,finner2002partitioning}. We discuss rank verification and compare it to our proposed inference scheme later in Section \ref{experiments}.

\section{Definitions and Problem Statement}\label{Definitions}

Let us now formulate our problem of interest. Let $\mathcal{X}$ be a collection of $|\mathcal{X}|=A$ algorithms. Let $p$ be a probability distribution $p$ over $\mathcal{X}$. We define $p_u$ as the probability that algorithm $u$ wins a future dataset. 
Let $X^n=X_1,...,X_n$ be a sample of $n$ independent observations from $p$. That is, $X_i=u$ corresponds to the event that algorithm $u$ won (attained the best result over) the $i^{th}$ dataset.  
Denote the sorted values of $p$ as $p_{[\cdot]}=p_{[1]}\geq p_{[2]}\geq ... \geq p_{[A]}$. Let $s=\argmax_{u \in \mathcal{X}}p_u$  be the most probable symbol in the sample, while $p_{s}=\max_{u \in \mathcal{X}}p_u=p_{[1]}$ is its corresponding probability. In our context, $s$ is the best-performing algorithm and $p_s$ is its win probability.  
Throughout this manuscript we assume that the maximum $s$ is unique. That is, there is a single algorithm that is most likely to win a future algorithm. We drop this assumption later in Section \ref{multiple maxima}. 

Let $N_u(X^n)=\sum_{i=1}^n \mathbbm{1}(X_i=u)$ be the number of wins of algorithm  $u$ in the sample. Let $\hat{p}$ be the maximum likelihood estimator of $p$. That is, $\hat{p}_u=N_u(X^n)/n$ for every $u \in \mathcal{X}$. Let $t=\argmax_{u \in \mathcal{X}}\hat{p}_u$  be the algorithm with the maximal number of wins in the sample, while $\hat{p}_{t}=\max_{u \in \mathcal{X}}\hat{p}_u=\hat{p}_{[1]}$ is its corresponding frequency. 
For the simplicity of the presentation we assume that $t$ is unique. We drop this assumption in Section \ref{multiple maxima} and show that our results still hold. 

In this work we seek a collection of algorithms that contains the best-performing algorithm with high confidence. Let us now formulate our objective using the general notation above. Given an alphabet $\mathcal{X}$ and a prescribed confidence level $1-\delta$ we seek a subset of symbols $\mathcal{I}_\delta(X^n) \subseteq \mathcal{X}$ such that 
\begin{align}\label{obj}
   \mathbb{P}( s\in \mathcal{I}_\delta(X^n))\geq 1-\delta.
\end{align}
Naturally, we may define the \textit{coverage set} $\mathcal{I}_\delta(X^n)$ as a collection of all the symbols $\mathcal{I}_\delta(X^n)=\mathcal{X}$ and the above holds. However, this solution is not interesting. Therefore, we seek a minimal size $\mathcal{I}_\delta(X^n)$, such that  $\mathbb{E}|\mathcal{I}_\delta(X^n)|$ is as small as possible. A natural choice for the desired objective is to define $\mathcal{I}_\delta(X^n)$ as the collection of the most frequent symbols in the sample   \citep{gupta1967selection,berger1980minimax}. That is, 
\begin{align}\label{CI}
\mathcal{I}_\delta(X^n)=\{u\;|\;\hat{p}_u\geq \hat{p}_{[1]}-D_\delta(X^n)\}.
\end{align} 
For example, $D_\delta(X^n)=0$ corresponds to choosing only the most frequent symbol, while $D_\delta(X^n)\geq \hat{p}_{[1]}$ corresponds to $\mathcal{I}_\delta(X^n)=\mathcal{X}$. Hence, our goal is to find $D_\delta(X^n)$ with a minimal expected length, $\mathbb{E}|D_\delta(X^n)|$,   such that (\ref{obj}) holds. 

A \textit{correct coverage} is defined as the event $\{s\in \mathcal{I}_\delta(X^n)\}$. That is, we successfully cover the most probable symbol $s$. This corresponds to $\hat{p}_s\geq \hat{p}_{[1]}-D_\delta(X^n)$ under the coverage set defined in (\ref{CI}). Therefore, our goal boils down to finding a minimal expected length $D_\delta(X^n)$ such that 
\begin{align}\label{CI2}
   \mathbb{P}( \hat{p}_{[1]}-\hat{p}_{s} \leq D_\delta(X^n))\geq 1-\delta.
\end{align}
We say that our inference scheme is \textit{exact} if (\ref{CI2}) holds with equality. This problem is widely known as subset selection over multinomial distributions.  

Perhaps the first major contribution to the problem is due to \cite{gupta1967selection}. In their work, \cite{gupta1967selection} introduced an exact solution to (\ref{CI2}) for a data-independent $D_\delta(X^n)=D_\delta$. In particular, they studied the worst-case distribution $p^*$ and derived a corresponding $D_\delta$ such that (\ref{CI2}) holds with equality. They characterized $p^*$ and provided a numerical procedure for $D_\delta$ that depends on $n$ and $A$. Unfortunately, their  procedure is computationally involved and may be applied to relatively small $n$ and $A$. \cite{berger1980minimax} extended the work of \cite{gupta1967selection}, showing that their procedure is minimax optimal. Further, he proposed  an alternative, simpler rule which performs a binomial test on each population, but its power does not necessarily increase as the number of observations increases \citep{hung2019rank}.
Later, \cite{gupta1989selecting} considered an empirical Bayes approach for selecting the best binomial population 
where a parametric prior distribution is assumed for the success probabilities for the different populations. \cite{ng2007selected} introduced an exact test for a modified problem in which the maximum count is fixed instead of
the total count $n$; that is, they sample until the leading symbol has at least $r$ counts, $\hat{p}_{[1]}\geq r$. Additional notable contributions are due to \cite{bose2001multinomial}, \cite{panchapakesan2006note} and \cite{bakir2013subset}. Among reference books on selection procedures we mention \cite{balakrishnan2007advances}, \cite{gibbons1999selecting}  and \cite{gupta2002multiple}. Yet to this day, \cite{gupta1967selection} remain the state of the art in the small sample -- small alphabet regimes \citep{hung2019rank}, while the complementary (large $n$, large $A$) regime remains an open challenge.   

It is important to mention that rank verification may be considered as a special case of the subset selection problem, as indicated by \cite{hung2019rank}. In particular, we can define a test that declares $v$ as a winner if $\mathcal{I}_\delta(X^n) = \{v\}$. Setting $\mathcal{I}_\delta(X^n)$ to satisfy (\ref{obj}), this test is valid at level $1-\delta$. Notice that the other direction does not apply. That is, subset selection can not be reformulated as a special case of rank verification, in the general setting.

\section{Main Results}
We now introduce our major contributions to the subset selection problem. We discuss its application to our inference problem later in Section \ref{experiments}. As mentioned above, our goal is to construct a minimal length CI, $D_\delta(X^n)$, such that $\hat{p}_{[1]}-\hat{p}_s \leq D_\delta(X^n)$ with high confidence. We distinguish between two regimes. First, we consider the asymptotic regime where $n$ is sufficiently large. Then, we study the finite sample regime where we make no assumptions on $n$. Finally, we provide a lower bound for the desired CI and demonstrate the tightness of our results.

\subsection{Asymptotic Regime}\label{asym regime section}
We begin our analysis with the following simple observation, 
 $$\hat{p}_{[1]}-\hat{p}_{s}=(\hat{p}_{[1]}-p_{s})+(p_{s}-\hat{p}_{s}).$$
 Therefore, it is enough to construct two simultaneous one-sided CIs to attain the desired coverage. That is, we seek minimal expected length CIs $U_{\delta/2}(X^n)$ and $V_{\delta/2}(X^n)$ such that 
\begin{align}\label{asym_cond}
    \mathbb{P}( \hat{p}_{[1]}-p_{[1]} \leq U_{\delta/2}(X^n))\geq 1-{\delta/2}\quad\quad\text{and}\quad\quad \mathbb{P}( {p}_{s}-\hat{p}_{s} \leq V_{\delta/2}(X^n))\geq 1-{\delta/2}   
\end{align}
to attain the desired $D_\delta(X^n)=U_{\delta/2}(X^n)+V_{\delta/2}(X^n)$. Fortunately, both $U_{\delta/2}$ and $V_{\delta/2}$ are straightforward to obtain. Let us begin with $V_{\delta/2}$. This is a classical one-sided binomial CI in the asymptotic regime. Applying Wald's CI \citep{brown2001interval} yields 
$$ V_{\delta/2}(X^n)=z_{\delta/2}\sqrt{\frac{p_s(1-p_s)}{n}}$$
where  $z_\delta/2$ is the upper $\delta/2$ quantile of the standard normal distribution. This results is known to be tight for $p_s$ that is not too close to $0$ or $1$ \citep{brown2001interval}. Next, we focus on $U_{\delta/2}$. \cite{xiong2009inference} show that $p_{[1]}-\hat{p}_{[1]} \stackrel{\cdot}{\sim} \mathcal{N}(0,\sqrt{p_{[1]}(1-p_{[1]})/n})$ for a unique maximum $p_{[1]}=\max_{u\in\mathcal{X}} p_u$. This again implies that  
$$ U_{\delta/2}(X^n)=z_{\delta/2}\sqrt{\frac{p_{[1]}(1-p_{[1]})}{n}}=V_{\delta/2}(X^n).$$ Putting it together, we obtain the following asymptotic result.

\begin{theorem}\label{asym regime}
Let $p$ be a distribution over $\mathcal{X}$. Let $X^n$ be a sample of $n$ independent observations from $p$. Let $\hat{p}$ be the MLE of $p$. Let $\mathcal{I}_\delta(X^n)=\{u\;|\;\hat{p}_u\geq \hat{p}_{[1]}-D_\delta(X^n)\}$ and 
\begin{align}
    D_\delta(X^n)=2z_{\delta/2}\sqrt{\frac{p_{[1]}(1-p_{[1]})}{n}}.
\end{align}
Then,  $s\in \mathcal{I}_\delta (X^n)$  with probability of at least $1-\delta$, for sufficiently large $n$.
\end{theorem}
In practice, we do not know $p_{[1]}$ so we replace it with its empirical counterpart $\hat{p}_{[1]}$ (following Wald's CI, for example). Importantly, it can be shown that $z_{\delta/2}$ behaves asymptotically like $\sqrt{2\log(2/\delta)}$ so our proposed CI is approximately 
\begin{align}
 D_\delta(X^n)=2\sqrt{2\log\left(\frac{2}{\delta}\right)\frac{\hat{p}_{[1]}(1-\hat{p}_{[1]})}{n}}
\end{align}
for sufficiently large $n$. 

\subsection{Finite Sample Regime}\label{finite regime section}
We now drop the large sample assumption and introduce a refined approach. As above, we first observe that
\begin{align}\label{basic finite}
  \hat{p}_{[1]}-\hat{p}_{s}=(\hat{p}_{[1]}-p_{t})+(p_{t}-p_{s})+(p_{s}-\hat{p}_{s})\leq (\hat{p}_{[1]}-p_{t})+(p_{s}-\hat{p}_{s})  
\end{align}
where the inequality is due to $p_{t}-p_{s}\leq 0$, following the definition of the symbol $s$. Therefore, it is enough to construct simultaneous one-sided CIs for $p_s$ (around $\hat{p}_s$) and $p_{t}$ (around $\hat{p}_{[1]}=\hat{p}_{t}$) to obtain the desired coverage. Formally, we seek minimal expected length CIs $W_{\delta}(X^n)$ and $V_{\delta}(X^n)$ such that 
\begin{align}\label{finite_condition}
    \mathbb{P}( \hat{p}_{t}-p_{t} \leq W_{\delta}(X^n))\geq 1-{\delta}\quad\quad\text{and}\quad\quad \mathbb{P}( {p}_{s}-\hat{p}_{s} \leq V_{\delta}(X^n))\geq 1-{\delta}
\end{align}
simultaneously hold, to obtain the desired $D_\delta(X^n)=W_{\delta}(X^n)+V_{\delta}(X^n)$. A possible approach is to apply a Bonferonni correction (union bound) and control both terms in a confidence level $1-\delta/2$, similarly to (\ref{asym_cond}). However, here we take a more refined approach. Notice that both $W_{\delta}(X^n)$ and $V_{\delta}(X^n)$ are CIs for multinomial parameters, $p_s$ and $p_t$, around their empirical counterparts, $\hat{p}_s$ and $\hat{p}_t$, respectively. Therefore, it is enough to consider the worst-case parameter, $\sup_{u\in \mathcal{X}}|p_u-\hat{p}_u|$ to control both cases simultaneously. Formally, assume we found $T_{\delta}(X^n)$ such that 
\begin{align}\label{infty norm}
\mathbb{P}(\sup_{u\in \mathcal{X}}|p_u-\hat{p}_u|\leq T_{\delta}(X^n))\geq 1-\delta.
\end{align}
Then, (\ref{finite_condition}) holds simultaneously by setting of $W_{\delta}(X^n)=U_{\delta}(X^n)=T_{\delta}(X^n)$. We show in the following that $T_\delta(X^n)$ mostly depends on $p$ through the maximal probability, $p_s$. In other words, the worst-case symbol in (\ref{infty norm}) is typically the most probable symbol $s$. This suggests that relaxing (\ref{finite_condition}) by controlling (\ref{infty norm}) is typically tight. We formalize this intuition later in Section \ref{lower_bound}. 

Let us now focus on (\ref{infty norm}). Notice that $\sup_{u\in \mathcal{X}}|p_u-\hat{p}_u|=||p-\hat{p}||_\infty$. Hence, our objective is to bound from above the MLE's convergence rate, with respect to the infinity norm. The infinity norm, also known as the 
\textit{uniform} or \textit{supremum} norm,
is a popular metric over distributions. It holds a number of important applications, in addition to being a fundamental object of independent interest \citep{boucheron2003concentration,van2014probability}.  Recently, \cite{kontorovich2024distribution} studied this exact problem. They introduced two main results for $T_\delta(X^n)$ which improve upon previously known alternatives. Their first result provides a competitive $T_\delta(X^n)$  with desirable properties. Unfortunately, its dependence in sample $X^n$ is quite involved and unintuitive (see Theorem $2$ of \cite{kontorovich2024distribution}). Their second result introduces a more direct dependency in the sample but is less competitive in practice (see Theorem $4$ of \cite{kontorovich2024distribution}). Specifically, they showed that with probability $1-\delta$
\begin{align}\label{aryeh}
      ||{p-\hat p}||_\infty \lesssim
\sqrt{\frac{\hat{p}_{[1]}(1-\hat{p}_{[1]})\log n}{n}+\frac{ \hat{p}_{[1]}(1-\hat{p}_{[1]})}{n}\log\frac1\delta}+\frac{1}{n}\log\frac{n}\delta+\frac{\log n}n
\end{align}
where $\lesssim$ hides small
absolute constants. Notice that the dependency in $n$ is  $O(\sqrt{\log{n}/n})$, as opposed to the desired $O(1/\sqrt{n})$ (as in Theorem \ref{asym regime}). We now revisit the infinity norm problem (\ref{infty norm}). We introduce a refined analysis that provides a typically smaller $T_\delta(X^n)$, and that approaches the desired asymptotic dependency (\ref{asym regime}). 

Before we proceed, we introduce some additional notation. Let $Y\sim\text{Bin}(n,\theta)$ be a Binomial random variable. 
\cite{skorski2020handy} showed that the $m^{th}$ central moment of $Y$ is a symmetric polynomial function of $\theta(1-\theta)$ when $m$ is even and antisymmetric when $m$ is odd. That is,
\begin{align}\label{moments}
 \mathbb{E}(Y-n\theta)^m=
    \begin{cases}
      \sum_{k=1}^{m/2} c_{k,m,n} \theta^k(1-\theta)^k & \text{m is even}\\
      \sum_{k=1}^{\lf m/2 \rf} c_{k,m,n}(1-2\theta)\theta^k(1-\theta)^k &  \text{m is odd}
    \end{cases}    
\end{align}
where $c_{k,m,n}$ are coefficients that depend on $k,n$ and $m$. Further, it can be shown that $c_{k,m,n}\leq k^{m-k}n^k$  for even $m$ (Theorem $4$ of \cite{skorski2020handy}). For example, the first six central moments of $Y$ are provided in Table \ref{central moments}. A longer list with higher moments appears in Table $2$ of \cite{skorski2020handy}. 

\begin{table}[h]
\begin{center}
\caption{First Six Central Moments of a 
Binomial Distribution}
\vspace{0.1cm}
\label{central moments}
\begin{tabular}{ |c||p{10cm}| }
 \hline
 $m$& \multicolumn{1}{|c|}{$\mathbb{E}(Y-n\theta)^m$}  \\
 \hline
 $1$   & 0   \\
 \hline
 $2$&   $n\theta(1-\theta)$  \\
 \hline
 $3$&   $n(1-2\theta)\theta(1-\theta)$  \\
 \hline
 $4$&   $3n(n-2)\theta^2(1-\theta)^2+n\theta(1-\theta)$
  \\
  \hline
  $5$&   $(1-2\theta)(10n(n-2)\theta^2(1-\theta)^2+n\theta(1-\theta))$
  \\
  \hline
  $6$&   $15(n^2-5n+10)\theta^3(1-\theta)^3+(15n-30)\theta^2(1-\theta)^2+n\theta(1-\theta)$
  \\
 \hline
\end{tabular}
\end{center}
\end{table}

Theorem \ref{T_data_independent} below introduces an upper bound to the infinity norm, which depends on the central moments of the Binomial distribution and holds for every even $m$. Its detailed proof is provided in Section \ref{A1} below. 
\begin{theorem}\label{T_data_independent}
     Let $p$ be a distribution over $\mathcal{X}$. Let $X^n$ be a sample of $n$ independent observations from $p$. Let $\hat{p}$ be the MLE of $p$.
      Then, with probability of at least $1-\delta$,
     \begin{align}\label{Data Independent T1}
            &\nrm{p-\hat p}_\infty\leq\frac{1}{n} \left(\frac{1}{\delta^{1/m}}\right) \left(\sum_{k=1}^{m/2} c_{k,m,n} \sum_{u\in\mathcal{X}}p_u^k(1-p_u)^k\right)^{1/m}
        \end{align}
        for every even $m$, where $c_{k,m,n}$ are the coefficients of the $m^{th}$ central moment of a binomial distribution (\ref{moments}).
\end{theorem}
Theorem \ref{T_data_independent} relies on Markov's inequality and the central moments of the infinity norm. It is a refined application of Theorem $1$ from  \cite{kontorovich2024distribution}. Importantly, (\ref{Data Independent T1}) introduces significantly tighter coefficients than of \cite{kontorovich2024distribution}, mostly for smaller values of $n$ and $m$. 

Next, we present Theorem \ref{main_result_1} which provides a data-dependent bound and replaces $p$ with its empirical counterpart $\hat{p}$. Its proof is located in Section \ref{proof2}.

\begin{theorem}\label{main_result_1}
    Let $\delta_1>0$ and $\delta_2>0$. Let $m$ be a positive even number. Then, with probability at least $1-\delta_1-\delta_2$,
    \begin{align}\label{11}
&\nrm{p-\hat p}_\infty\leq R_{\delta_1,\delta_2,m}(X^n)=\frac{1}{n}\sqrt{\frac{n}{n-1}}
\left(\frac{1}{\delta_1}\left( \sum_{u \in\mathcal{X}}\sum_{k=1}^{m/2}c_{k,m,n}(\hat{p}_u(1-\hat{p}_u))^k+\epsilon_{n}\right)\right)^{1/m}
\end{align}
for every even $m$, where
\begin{align}\nonumber
\epsilon_n=\sqrt{\frac{2}{n}\log\left(\frac{1}{\delta_2}\right)}\left(\sup_{p \in [0,1]}\sum_{k=1}^d c_{k,m,n}k\left(p(1-p)\right)^{k-1}(1-2p)+\sum_{k=1}^d\frac{c_{k,m,n}k(k-1)}{n\cdot 2^{2k-3}}\right).
\end{align}
\end{theorem}
To prove Theorem \ref{main_result_1} we utilize Theorem  (\ref{Data Independent T1}) with $\delta=\delta_1$. Then, we apply McDiarmid's inequality to obtain a concentration bound for $\sum_{k=1}^{m/2} c_{k,m,n} \sum_{u\in\mathcal{X}} p_u^k(1-p_u)^k$ around its empirical counterpart, with probability $1-\delta_2$. Finally, we apply the union bound to obtain the stated result. To further clarify the proposed bound we introduce the following simplified corollary, whose proof is located in Section \ref{proof3}.

\begin{corollary}\label{T2_data_dependent}
The upper bound in Theorem \ref{main_result_1} satisfies
     \begin{align}\label{finite}
    &R_{\delta_1,\delta_2,m}(X^n)\leq\frac{1}{n}\left(\frac{c_{d,m,n}}{\delta_1}\right)^{1/m}\left( \sum_{u\in\mathcal{X}} (\hat{p}_u(1-\hat{p}_u))^{m/2}\right)^{1/m}+O\left(\frac{1}{n^{0.5(1+1/m)}}\right).
\end{align}
Furthermore, for a choice of $m^* = 2\log(1/\delta_1)$ we have  
     \begin{align}\nonumber
    &R_{\delta_1,\delta_2,m^*}(X^n)\leq\sqrt{\frac{\exp(1)\log(1/\delta_1)}{n}}\left( \sum_{u\in\mathcal{X}} (\hat{p}_u(1-\hat{p}_u))^{m^*/2}\right)^{1/m^*}+O\left(\frac{1}{n^{0.5(1+1/m^*)}}\right).
\end{align}
\end{corollary}
Notice that the choice of $m$ is an application of a  Chernoff-type inequality, where the bound in (\ref{finite}) is minimized with respect to the order of the central moment $m$. 

Importantly, we show that  
$\left( \sum_u (\hat{p}_u(1-\hat{p}_u))^{m/2}\right)^{1/m}$ approximates $\sup_u \sqrt{\hat{p}_u(1-\hat{p}_u)}=\sqrt{\hat{p}_{[1]}(1-\hat{p}_{[1]})}$ as $m$ increases (see Section \ref{labeled_appendix}). This allows the desired dependency in $\hat{p}$, as appears in Theorem \ref{asym regime}. Finally, notice that the obtained bound only depends on $\delta_1$ as the term that depends on $\delta_2$ is practically negligible. Theorem \ref{finite sample regime} below concludes our proposed scheme for the finite sample regime.

\begin{theorem}\label{finite sample regime}
Let $p$ be a distribution over $\mathcal{X}$. Let $X^n$ be a sample of $n$ independent observations from $p$. Let $\hat{p}$ be the MLE of $p$. Let $\delta_1+\delta_2=\delta$
and $T_\delta(X^n)=2R_{\delta_1,\delta_2,m}$. Define
$$\mathcal{I}_\delta(X^n)=\{u\;|\;\hat{p}_u\geq \hat{p}_{[1]}-T_\delta(X^n)\}.$$
Then,  $s\in \mathcal{I}_\delta (X^n)$  with probability of at least $1-\delta$, for every even $m>0$.
\end{theorem}

Let us compare Theorem  \ref{finite sample regime} with the 
asymptotic result in Theorem \ref{asym regime}. First, we 
observe that both results exhibit the same dependency in $n$ and $\delta$, for the recommended choice of $m$. Further, we notice that both schemes depend on the underlying distribution in (approximately) the same form, $\sqrt{p_{[1]}(1-p_{[1]})}$. This again emphasizes the tightness of our finite sample result.  

\section{A Lower Bound}
Finally, we introduce a lower bound for our purposed scheme. Theorem \ref{lower_bound} below, whose proof is located in Appendix \ref{lower bound proof}, shows that for every $n$ there exist distributions $p$ for which our proposed schemes are asymptotically tight. 

\begin{theorem}\label{lower_bound}
    Assume that $p_{[1]}-p_{[2]}=O(1/n)$. Suppose there exists $T$ such that 
    $\mathbb{P}( \hat{p}_{[1]}-\hat{p}_{s} \leq T_\delta)\geq 1-{\delta}.$
    Then,  
    $$T_\delta\geq T_{min}=z_\delta \sqrt{\frac{2\left(p_{[1]}(1-p_{[1]})-p_{[1]}^2\right)}{n}}+O\left(\frac{1}{n}\right)$$
    for sufficiently large $n$. Furthermore, assume  $p_s\leq1/3$. Then, 
    $$T_{min}\geq z_\delta \sqrt{\frac{p_{[1]}(1-p_{[1]})}{n}}+O\left(\frac{1}{n}\right).$$
    
\end{theorem}
Theorem \ref{lower_bound} introduces a lower bound for challenging distributions where the second most probable symbol $p_{[2]}<p_{[1]}$ is close enough to $p_{[1]}$. This result is similar in spirit to \cite{gupta1967selection} who study the worst case distribution $p$ for (\ref{CI}) and discuss its minimax guarantees. Hence, we show that for every $n$, there exist worst-case distributions such that  there is no CI shorter than $T_{min}$ without violating the prescribed confidence level. Our proposed schemes approximately attain $T_{min}$ (up to constants), demonstrating their optimality. 

\section{The Multiple Maxima Case} \label{multiple maxima}
Throughout the manuscript we assume that both $\hat{p}_{[1]}$ and $p_{[1]}$ are unique. We now drop these assumptions and revisit our results. We begin with the simpler case where $p_{[1]}$ is still unique, but $\hat{p}_{[1]}$ is not. That is, there is only a single maximum to $p$, but the multiple symbols attain the maximal frequency in the sample, $\hat{p}_{[1]}$. Formally, we denote this collection of symbols as $\mathcal{T}=\{t|\;\hat{p}_t\geq \hat{p}_v\;\; \forall t,v  \in \mathcal{X}\}$. Notice that both our objective (\ref{obj}) and the desired coverage set (\ref{CI}) remain unchanged. That is, we seek minimal expected length $D_\delta(X^n)$ such that $\hat{p}_{[1]}-\hat{p}_s \leq D_\delta(X^n)$ with high confidence, regardless of $|\mathcal{T}|$. 
In the asymptotic regime (Section \ref{asym regime section}) we show that $\hat{p}_{[1]}-\hat{p}_s=(\hat{p}_{[1]}-p_s)+(p_s-\hat{p}_s)$ and derive two simultaneous one-sided CIs for both terms. The same derivation also applies to our current setup. That is, the asymptotically normal distribution of the first term still holds, even if $\hat{p}_{[1]}$ is not unique (see \cite{xiong2009inference}), while the second term remains unchanged. This means that Theorem \ref{asym regime} also holds in this case. Let us proceed to the finite sample regime. Once again we follow (\ref{basic finite}) to obtain 
$\hat{p}_{[1]}-\hat{p}_{s}\leq (\hat{p}_{[1]}-p_{t})+(p_{s}-\hat{p}_{s})$, where this time $t\in\mathcal{T}$ and $\hat{p}_{[1]}=\hat{p}_t$ is the corresponding (maximal) frequency of $t\in \mathcal{T}$ in the sample. Notice that by controlling $\sup_{u\in\mathcal{X}}|p_u-\hat{p}_u(X^n)|$ we control the absolute difference for every symbol in $\mathcal{X}$, including $t\in \mathcal{T}\subseteq\mathcal{X}$. This means that by constrution, Theorem \ref{finite sample regime} also holds for the case where $t$ is not unique.  

Let us now discuss the more challenging case where $s$ is not unique. Denote the set of most probable symbols with $\mathcal{S}$ and $p_{s}=p_{[1]}=\max_{u \in \mathcal{X}}p_u$ is again the maximal probability over $p$. In this case, our object becomes $\mathbb{P}( S\subseteq \mathcal{I}_\delta(X^n))\geq 1-\delta$ while the definition of $\mathcal{I}_\delta(X^n)$ remains unchanged (\ref{CI}). This means we seek minimal expected length $T_\delta(X^n)$ such that $\hat{p}_{[1]}-\hat{p}_s \leq D_\delta(X^n)$ for every $s\in\mathcal{S}$ with high confidence. Formally, we require that $$\mathbb{P}(\cap_{s\in\mathcal{S}}\{\hat{p}_{[1]}-\hat{p}_s\leq D_\delta(X^n)\})\geq1-\delta.$$
Unfortunately, we cannot apply our asymptotic results for this setup. Specifically, the normal approximation of $\hat{p}_{[1]}-p_s$ only applies to a unique maximum regime \citep{xiong2009inference}. In fact, the distribution of $\hat{p}_{[1]}-p_s$ in the presence of multiple maxima depends on the cardinality of $\mathcal{S}$, which is typically unknown and difficult to estimate \citep{xiong2009inference}. However, our finite sample result still hold. Specifically, following (\ref{basic finite}) we require that 
$$\hat{p}_{[1]}-\hat{p}_{s}\leq (\hat{p}_{[1]}-p_{t})+(p_{s}-\hat{p}_{s})\leq D_\delta(X^n)$$
would hold for every $s\in\mathcal{S}$ with high confidence. However, by controlling $\sup_{u\in\mathcal{X}}|p_u-\hat{p}_u(X^n)|$ we control the absolute difference for all the symbols in $\mathcal{X}$, including $s\in \mathcal{S}\subseteq\mathcal{X}$. Once again, it means that Theorem \ref{finite sample regime} holds for this challenging regime. 

To conclude, the asymptotic result (Theorem \ref{asym regime}) only holds for a unique $s$. However, it does apply to the case where the most frequent symbol in the sample is not unique. On the other hand, the more robust finite sample regime holds for both cases. This makes it a more adequate choice for the general case, where we cannot assume that the maximum of $p$ is not unique. Finally, the lower bound \ref{lower_bound} trivially holds for both setups. 

\section{Experiments}\label{experiments}

Let us now illustrate our results in synthetic and real-world experiments. 

\subsection{Synthetic Experiments}

First, we study two example distributions $p$, which are common benchmarks for probability estimation and related problems \citep{orlitsky2015competitive}. The Zipf's law distribution is a typical heavy-tailed benchmark in probability estimation; it is commonly used for modeling natural (real-world) quantities in physical and social sciences, linguistics, economics and others fields \citep{saichev2009theory}. The Zipf's law distribution follows $P(u;s,A)={u^{-s}}/{\sum_{v=1}^A v^{-s}}$ where $A$ is the alphabet size and $s$ is a skewness parameter. In our context, the heavy-tailed Zipf's law distribution corresponds to a case where a (very) small number of algorithms typically outperform all the others. On the other edge, we examine a \textit{near uniform} distribution. Specifically, we uniformly draw a distribution $p$ from the simplex $\mathcal{P}=\{p\;|\sum_up_u=1,\;p_u\geq0  \}$. This setup corresponds to the case where all algorithms perform almost equally. 

In our first experiment we study $m=20$ competing algorithms and draw $n$ samples from $p$. Notice that each sample corresponds to the winning (best performing) algorithm of the $m$ candidates, over each of the $n$ datasets. Given the $n$ samples, we construct a confidence set for $p_{[1]}$, assuming a finite sample (Section \ref{finite regime section}) and an asymptotic regime (Section \ref{asym regime}).  We repeat this process $1000$ times (that is, draw $n$ samples and construct a corresponding $\mathcal{I}_\delta(X^n)$) to attain an estimate of $\mathbb{E}|\mathcal{I}_\delta(X^n)|$ and a coverage rate of $p_{[1]}$. We compare our results to an Oracle who knows the most probable symbol $s$, but is restricted to a confidence set of the form \ref{CI2}. This Oracle is even stronger than our proposed lower bound in Theorem \ref{lower_bound}, as it provides an exact yet unattainable confidence set (due to its knowledge of $s$). Notice that both \cite{gupta1967selection} and \cite{berger1980minimax} are not applicable for our setup due to the relatively large alphabet and sample size. Figure \ref{fig1} demonstrates the results we achieve for the Zipf's law (left) and the near uniform distribution (right) as $n$ increases.

\begin{figure}[ht]
\centering
\includegraphics[width =0.8\textwidth,bb= 25 250 550 590,clip]{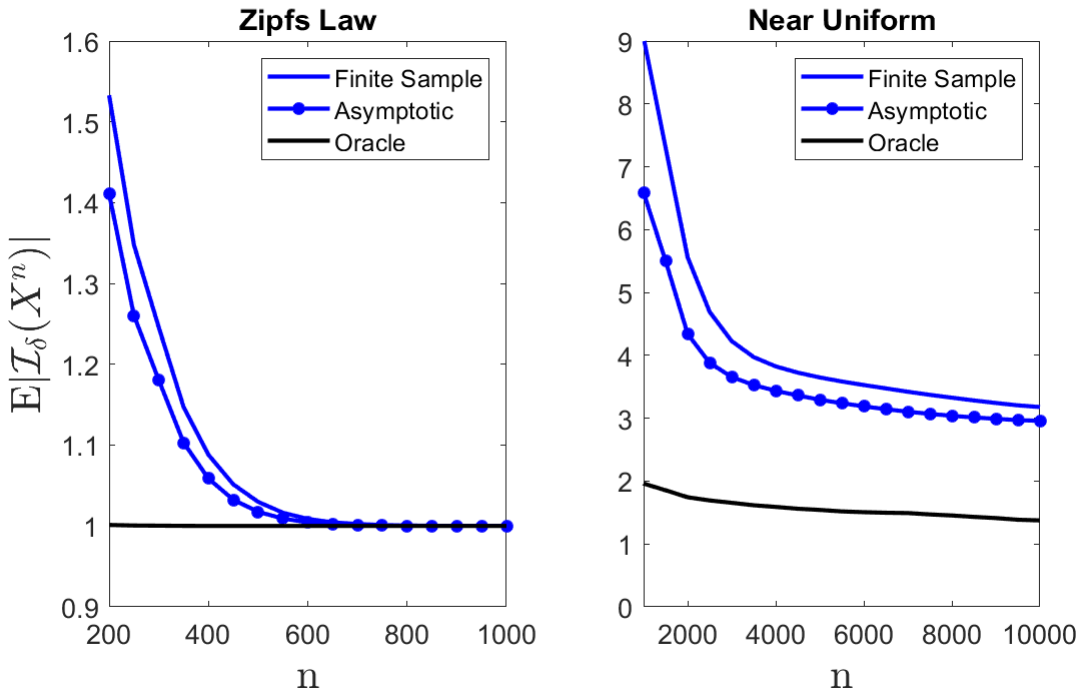}
\caption{The proposed estimators}
\label{fig1}
\end{figure}

First, we observe that both Theorem \ref{finite sample regime} and Theorem \ref{asym regime} schemes converge to the Oracle's performance for relatively small $n$ in the Zipf's law experiment. The reason is fairly straightforward - in this setup $p_{[1]}=0.278$ while $p_{[2]}$ and $p_{[3]}$ are $ 0.139$ and $0.092$ respectively. Hence, the difference between the best-performing algorithm and its runner-ups is relatively large, which makes it easier to be distinguish from the sample. This also explains the optimal performance of the Oracle, which attains a confidence set of size one on average (that is, with $p_{[1]}$ only). Next, we examine the near uniform results. Here, the situation is more challenging and both proposed methods converge more slowly to the Orcale. This is not quite surprising, given that the leading probabilities are $0.094, 0.092, 0.091$ respectively.  Yet, the attained confidence set is of a reasonable size ($\approx 4$ candidate algorithms) compared to the Oracle (with an average size of two candidates). 

\subsection{Real-world Experiments}

We now proceed to real-world experiments. We consider large experimental studies, where multiple key ML algorithms are compared against each other over a benchmark of datasets. 

First, we consider the work of  \cite{fernandez2014we}. In this work, the authors evaluate a total of $179$ classifiers (arising from $17$ basic models), implemented on Weka, R, C and Matlab. They use $121$ data sets, which represent the entire UCI repository and additional real problems, in order to achieve significant conclusions about the classifiers' behavior. They conclude that the classifiers most likely to be the best are the Random Forest (RF) versions. However, the gap with the second best is not statistically significant. We apply our proposed inference scheme to the results of \cite{fernandez2014we}, to identify the candidates for the best-performing classifier. First, notice that the study of Delgao et al. compares an extremely large number of algorithms, which even exceeds the number of datasets. Furthermore, many of the algorithms are practically identical, as they only differ in the platform that they were implemented on. Hence, analyzing the winner among this set of algorithms is very noisy and prune to random results. Therefore, as a first step in reducing unnecessary comparisons, we only study one version of each algorithms. That is, we do not compare the same algorithm, implemented on two different platforms. For simplicity, we focus on the $\text{R}$ classifiers, which consists of $A=36$ algorithms.
The three top performing algorithms according to the MLE criterion are \textit{Random Forest}, \textit{adaBoost} and \textit{K-nn} with estimated probabilities of $0.2564$, $0.1368$, and  $0.0769$, respectively. Applying our proposed inference scheme (Theorem \ref{finite sample regime}), we argue that  \textit{Random Forest} and \textit{ AdaBoost} are the only candidates for the best-performing algorithm, with a confidence level of $0.95$. Further, we apply \citep{hung2019rank} rank verification test to the studied problem. Interestingly, the test does not reject any of the hypotheses, meaning that no ranking is statistically significant in a level of $0.95$. This emphasizes the difference between subset selection and rank verification. Specifically, our proposed subset selection scheme implies that the difference between \textit{Random Forest} and \textit{adaBoost} is not significant, but the difference between \textit{Random Forest} and \textit{K-nn} is. On the other hand, the rank verification only examines the differences between consecutive ranks and concludes that no consecutive rankings are significant. Finally, we apply statistical comparisons as proposed by \cite{demvsar2006statistical}. Following \cite{demvsar2006statistical} recommendation, we first test the hypothsis that all algorithms perform equally well. This test is easily rejected with a negligible p-value. Next, we apply the post-hoc Nemeyni test \citep{nemenyi1963distribution}. This rank-based test shows that $28$ of the $36$ algorithms are not statistically significant worse than the top performing algorithm. This conclusion is not quite surprising, given the relatively low power of this test, as \cite{demvsar2006statistical} indicates. 

Next, we turn to comparative study of  \cite{shmuel2024comprehensive}. In this work, the authors introduce a comprehensive benchmark, aimed at better characterizing the types of datasets where deep learning models excel. They evaluate $111$ datasets with $20$ different models, including both regression and classification tasks. Here, we only focus on the regression analysis with $A=20$ regressors and $n=57$ datasets. The top performing (most wins) algorithms are  \textit{AutoGluon}, \textit{SVM} and \textit{ResNet} with corresponding estimated probabilities of $0.3514$, $0.0901$ and $0.0811$ respectively. Applying our proposed scheme we conclude that the only candidate for best-performing algorithm is \textit{AutoGluon}. This result is consistent with \cite{hung2019rank}, who shows that the difference between \textit{AutoGluon} and \textit{SVM} is statistically significant, but cannot conclude the same for \textit{SVM} and \textit{ResNet}. Further, we apply \cite{demvsar2006statistical} recommendation and reject the null that all algorithms perform equally well. As in the previous example, we proceed with the Nemeyni test. Unfortunately, this test is not powerful enough, with $8$  algorithms that are comparable (not significantly inferior) to the top performing \textit{AutoGluon} algorithm. 

Finally, we revisit the work of \cite{mcelfresh2024neural}. In this work, the authors introduce a benchmark of $176$ tabular datasets, and compare a total of $19$ popular ML algorithms. In addition, they introduce the \textit{TabZilla} Benchmark Suite, which is  a collection of the $36$ ''hardest" of their datasets. The performance of the studied algorithms over TabZilla is provided in Table $4$ of \cite{mcelfresh2024neural}, where they report the performance of the top three algorithms for each dataset. The conclusions of this study are not entirely decisive. The top performing algorithms (according to TabZilla) are CatBoost, XGBoost and LightGBM with  corresponding estimates of $0.2667,0.2667$ and $0.2333$, respectively. Applying our suggested inference scheme to these result, we get that the confidence set $\mathcal{I}_\delta(X^n)$ contains all the $20$ algorithms. In other words, we cannot provide an informative subset of best-performing algorithms. This result aligns with \cite{hung2019rank}, which does not reject any of the null hypotheses. That is, none on the rankings are statistically significant. Finally, we turn to \cite{demvsar2006statistical} recommendation. Here, we cannot reject the null hypothesis which argues that all algorithms perform equally well (with a p-value of almost $1$). Hence, all three schemes agree that no statistically significant conclusions may be drawn from this dataset.

\section{Conclusion}
In this work we introduce a novel inference scheme for the best-performing algorithm. Given the performance of $A$ algorithms over a benchmark of $n$ datasets, our goal is to construct a minimal size subset of algorithms that contains the best-performing algorithm with high confidence. We formulate this problem as subset selection for multinomial distributions and introduce a near optimal solution to the problem. We distinguish between the asymptotic and finite sample regimes, and a matching lower bound which validates the tightness of our results. Finally, we demonstrate the favorable performance of our proposed inference schemes in synthetic and real-world experiments. 

The focus of this work is comparative studies in the machine learning community. Yet, our contribution is not limited to this application of interest. That is, every problem that may be formulated as subset selection for multinomial distributions may benefit from our proposed inference scheme. For example, consider a poll with $A$ candidates, where our goal is to identify the most popular candidate in the population. Our proposed scheme introduces a powerful and statistical valid tool for this exact purpose. We consider additional applications and their required adaptions for our future work.

\section{Proofs}
\subsection{Proof of Theorem \ref{T_data_independent}}\label{A1}
First, notice we have \begin{align}\label{main_data_depented_tight}
    \mathbb{E}\big(\sup_{u\in\mathcal{X}} |p_u-\hat{p}_u(X^n)|\big)^m\overset{(\text{i})}{=}&\mathbb{E}\left(\sup_{u\in\mathcal{X}} (p_u-\hat{p}_u(X^n))^m\right)\overset{(\text{ii})}{\leq}\mathbb{E}\left(\sum_{u\in\mathcal{X}} (p_u-\hat{p}_u(X^n))^m\right)=\\\nonumber
    &\frac{1}{n^m}\sum_{u\in\mathcal{X}}\mathbb{E}(N_u-np_u)^m\overset{(\text{iii})}{=}\frac{1}{n^m} \sum_{u\in\mathcal{X}} \sum_{k=1}^d c_{k,m,n}(p_u(1-p_u))^k
    \end{align}
    where $d=m/2$ and
    \begin{enumerate}[(i)]
    \item follows  from the monotonicity of the power function.
    \item The supremum of non-negative elements is bounded from above by their sum \citep{maddox1988elements}.   
    \item By the definition of the $k^{th}$ central moment of a binomial distribution \citep{skorski2020handy}. 
    \end{enumerate}
    Applying Markov's inequality we obtain
    \begin{align}
    \P\big(\sup_{u\in\mathcal{X}} |p_u-\hat{p}_u(X^n)|\geq a\big)\leq &\frac{1}{a^m}\mathbb{E}\left(\sup_{u\in\mathcal{X}} |p_u-\hat{p}_u(X^n)|\right)^m = \\\nonumber
    &\frac{1}{a^m}\frac{1}{n^m} \sum_{u\in\mathcal{X}} \sum_{k=1}^d c_{k,m,n}(p_u(1-p_u))^k.
    \end{align}
    Setting the right hand side to equal $\delta$ yields 
    $$a=\frac{1}{n}\left(\frac{1}{\delta_1} \sum_{u\in\mathcal{X}} \sum_{k=1}^d c_{k,m,n}(p_u(1-p_u))^k\right)^{1/m}. $$
Therefore, with probability $1-\delta$, we have
\begin{align}\label{ba}
    \sup_{u\in\mathcal{X}} |p_u-\hat{p}_u&(X^n)|\leq \frac{1}{n}\left(\frac{1}{\delta}  \sum_{k=1}^d c_{k,m,n}\sum_{u\in\mathcal{X}} p_u^k(1-p_u)^k\right)^{1/m}.
    \end{align}

\subsection{Proof of Theorem \ref{main_result_1} }\label{proof2}
We begin with the following proposition. 
\begin{proposition}\label{P2_tight}
        Let $\delta_2>0$. Then, with probability $1-\delta_2$, 
\begin{align}
    \sum_{u\in\mathcal{X}}\sum_{k=1}^{m/2} k^{m-k} (&n(p_u(1-p_u))^k\leq \left(\frac{n}{n-1}\right)^d\bigg( \sum_{u\in\mathcal{X}}\sum_{k=1}^{m/2} 
  c_{k,m,n}(\hat{p}_u(1-\hat{p}_u))^k+\epsilon_{n} \bigg)
\end{align}
for every even $m$, where 
\begin{align}
\epsilon_n=\sqrt{\frac{2}{n}\log\left(\frac{1}{\delta_2}\right)}\left(\sup_{p \in [0,1]}\sum_{k=1}^d c_{k,m,n}k\left(p(1-p)\right)^{k-1}(1-2p)+\sum_{k=1}^d\frac{c_{k,m,n}k(k-1)}{n\cdot 2^{2k-3}}\right).
\end{align}
\end{proposition}
\begin{proof}
    Define $\psi(n,d,\hat{p})=\sum_{u\in\mathcal{X}}\sum_{k=1}^dc_{k,m,n}(\hat{p}_u(1-\hat{p}_u))^k$.
    McDiarmid's inequality yields 
    \begin{align}\nonumber   
    \text{P}\left(\psi(n,d,\hat{p})-\mathbb{E}\left(\psi(n,d,\hat{p})\right)\leq -\epsilon_n  \right)\leq \exp\left(\frac{-2\epsilon_n^2}{\sum_{j=1}^n c_j^2} \right)
    \end{align}
    where $c_j$ is defined as 
    \begin{align}
        \sup_{x'_j}\big|\psi(n,d,\hat{p})-\psi(n,d,\hat{p}')\big|\leq c_j.
    \end{align}
    Here, $\hat{p}'$ is the MLE over the same sample $x^n$, but with a different $j^{th}$ observation, $x_j'$.
    First, let us find $c_j$. We have
    \begin{align}\label{c_j_tight_pre}
    &\sup_{x'_j }\big|\psi(n,d,\hat{p})-\psi(n,d,\hat{p}')\big|\leq\\\nonumber 
    &\sup_{p\in [0,1-1/n]}2\big|\sum_{k=1}^dc_{k,m,n}\left(p(1-p)\right)^k-\sum_{k=1}^dc_{k,m,n}\left((p+1/n)(1-(p+1/n))\right)^k)\big|
    \end{align}
independently of $j$, where the inequality follows from the fact that changing a single observation effects only two symbols (for example, $\hat{p}_l$ and $\hat{p}_t$), where the change is $\pm 1/n$. 

Now, we would like to bound from above (\ref{c_j_tight_pre}). Denote $f(p)=\sum_{k=1}^dc_{k,m,n}p^k(1-p)^k$. Applying Taylor series to $f(p+1/n)$ around $f(p)$ yields
$$f\left(p+\frac{1}{n}\right)=f(p)+\frac{1}{n}f'_k(p)+r(p)$$
where $r(p)=\frac{1}{2!}\frac{1}{n^2}f''(c)$ is the residual and $c \in [p, p+1/n]$ \citep{stromberg2015introduction}. 
We have
\begin{align}\label{derivatives}
    &f'(p)=\sum_{k=1}^dc_{k,m,n}k\left(p(1-p)\right)^{k-1}(1-2p)\\\nonumber
    &f''(p)=\sum_{k-1}^dc_{k,m,n}k(k-1)(p(1-p))^{k-2}(1-2p)^2-2k(p(1-p))^{k-1}\leq\\\nonumber
    &\quad\quad\quad\;\sum_{k-1}^dc_{k,m,n}k(k-1)(p(1-p))^{k-2}.
\end{align}
Hence,
\begin{align}\label{Appendix A replaced}
&\sup_{p\in [0,1-1/n]}2\big|\sum_{k=1}^dc_{k,m,n}\left(p(1-p)\right)^k-\sum_{k=1}^dc_{k,m,n}\left((p+1/n)(1-(p+1/n))\right)^k)\big|=\\\nonumber
&\sup_{p \in [0,1-1/n]}2\big|-\frac{1}{n}f'(p)-\frac{1}{2!}\frac{1}{n^2}f''(c)\big|\leq\sup_{p \in [0,1-1/n]}\frac{2}{n}\big|f'(p)\big|+\frac{2}{2!}\frac{1}{n^2}\big|f''(c)\big|\overset{(\text{i})}{\leq}\\\nonumber
&\sup_{p \in [0,1]}\frac{2}{n}\sum_{k=1}^d c_{k,m,n}k\left(p(1-p)\right)^{k-1}(1-2p)+\sup_{p \in [0,1]}\frac{1}{n^2}\sum_{k=1}^dc_{k,m,n}k(k-1)(p(1-p))^{k-2}\overset{(\text{ii})}{\leq}\\\nonumber
&\sup_{p \in [0,1]}\frac{2}{n}\sum_{k=1}^d c_{k,m,n}k\left(p(1-p)\right)^{k-1}(1-2p)+\sum_{k=1}^d\frac{c_{k,m,n}k(k-1)}{n^2\cdot 4^{k-2}}
\end{align}
where 
\begin{enumerate}
    \item follows from (\ref{derivatives}).
    \item follows from the concavity and the symmetry of $\left(p(1-p)\right)^{k}$ for $k\geq 1$. 
\end{enumerate}
Denote 
$$\phi_{n,d}=\sup_{p \in [0,1]}\frac{2}{n}\sum_{k=1}^d c_{k,m,n}k\left(p(1-p)\right)^{k-1}(1-2p)+\sum_{k=1}^d\frac{c_{k,m,n}k(k-1)}{n^2\cdot 4^{k-2}}.$$
Therefore, we have 
\begin{align}\label{c_j_tight}
    &\sup_{x'_j}\big|\psi(n,d,\hat{p})-\psi(n,d,\hat{p}')\big|{\leq} \phi_{n,d}.
    \end{align}
Next, 
\begin{align}\nonumber
\mathbb{E}(\psi(n,d,\hat{p}))\geq& \sum_{u\in\mathcal{X}}\sum_{k=1}^d c_{k,m,n}\left(\mathbb{E}(\hat{p}_u(1-\hat{p}_u)) \right)^k=\\\nonumber
&\sum_{u\in\mathcal{X}}\sum_{k=1}^d c_{k,m,n} \left(\left(1-\frac{1}{n}\right)p_u(1-p_u)\right)^k\geq\\
& \left(1-\frac{1}{n}\right)^d\sum_{u\in\mathcal{X}}\sum_{k=1}^d c_{k,m,n}p_u(1-p_u))^k=\left(1-\frac{1}{n}\right)^d\psi(n,d,p)\label{V1_tight}
\end{align}
where the first inequality follows from Jensen Inequality, the equality that follows is due to $\mathbb{E}(\hat{p}_i(1-\hat{p}_i))=\mathbb{E}(\hat{p}_i)-\text{Var}(\hat{p}_i)-\mathbb{E}^2(\hat{p}_i)=p(1-p)(1-1/n)$, and the last inequality follows from $(1-1/n)^k\geq(1-1/n)^d$ for $1\leq k\leq d$. Going back to McDiarmid's inequality, we have 
\begin{align}\label{ineq1} \P \left(\mathbb{E}\psi(n,d,\hat{p})\geq \psi(n,d,\hat{p})+\epsilon_n  \right)\leq \exp\left(\frac{-2\epsilon_n^2}{n \phi_{n,d}^2} \right)
\end{align} 
In word, the probability that the random variable $Z=\psi(n,d,\hat{p})$ is smaller than a constant $C=\mathbb{E}(\psi(n,d,\hat{p}))-\epsilon_n$ is not greater that $\nu=\exp\left(-2\epsilon^2/{n\phi_{n,d}^2} \right)$. Therefore, it necessarily means that the probability that $Z$ is smaller than a constant smaller than $C$, is also not greater than $\nu$. Hence, plugging  (\ref{V1_tight})  we obtain
\begin{align}\nonumber
\P\left((1-1/n)^d\psi(n,d,p)\geq \psi(n,d,\hat{p})+\epsilon_n  \right)\leq \exp\left(\frac{-2\epsilon_n^2}{n\phi_{n,d}^2} \right)
\end{align}
Setting the right hand side to equal $\delta_2$ we get 
\begin{align}
\epsilon_n=&\sqrt{\frac{n}{2}\log\left(\frac{1}{\delta_2}\right)}\phi_{n,d}=\\\nonumber
&\sqrt{\frac{2}{n}\log\left(\frac{1}{\delta_2}\right)}\left(\sup_{p \in [0,1]}\sum_{k=1}^d c_{k,m,n}k\left(p(1-p)\right)^{k-1}(1-2p)+\sum_{k=1}^d\frac{c_{k,m,n}k(k-1)}{n\cdot 2^{2k-3}}\right).
\end{align} and with probability $1-\delta_2$, 
\begin{align}
    \sum_{u\in\mathcal{X}}\sum_{k=1}^d c_{k,m,n}& (p_u(1-p_u))^k\leq\left(\frac{n}{n-1}\right)^d \left( \sum_{u\in\mathcal{X}}\sum_{k=1}^d 
  c_{k,m,n}(\hat{p}_u(1-\hat{p}_u))^k+\epsilon_n \right).
\end{align}.
\end{proof}
Finally, we apply the union bound to (\ref{ba}) with $\delta=\delta_1$ and Proposition  \ref{P2_tight} to obtain the stated Theorem $1$.

\subsection{A Proof for Corollary \ref{T2_data_dependent}}\label{proof3}
Let us examine the terms of Theorem \ref{main_result_1}. 
First, notice that $c_{k,m,n}=O(n^k)$ \citep{skorski2020handy}. Therefore, 
\begin{align}
\sum_{u\in\mathcal{X}}\sum_{k=1}^{d}c_{k,m,n}\hat{p}_u^k(1-\hat{p}_u)^k=c_{d,m,n}\sum_{u\in\mathcal{X}}\hat{p}_u^d(1-\hat{p}_u)^d+O\left(n^{d-1}\right).
\end{align}
Further, we have that
\begin{align}
\epsilon_n=&\sqrt{\frac{2}{n}\log\left(\frac{1}{\delta_2}\right)}\left(\sup_{p \in [0,1]}\sum_{k=1}^d c_{k,m,n}k\left(p(1-p)\right)^{k-1}(1-2p)+\sum_{k=1}^d\frac{c_{k,m,n}k(k-1)}{n\cdot 2^{2k-3}}\right)=\\\nonumber
&\sqrt{\frac{2}{n}\log\left(\frac{1}{\delta_2}\right)}\left(\sup_{p \in [0,1]} c_{d,m,n}d\left(p(1-p)\right)^{d-1}(1-2p)+O(n^{d-1})\right)=\\\nonumber
&\left(\frac{c_{d,m,n}}{\sqrt{n}}\right)\sqrt{2\log\left(\frac{1}{\delta_2}\right)}\sup_{p \in [0,1]} d\left(p(1-p)\right)^{d-1}(1-2p)+O(n^{d-3/2})=O\left(n^{d-1/2}\right).
\end{align}
Hence, 
\begin{align}
    \left(\sum_{u\in\mathcal{X}}\sum_{k=1}^{d}c_{k,m,n}\hat{p}_u^k(1-\hat{p}_u)^k+\epsilon_n\right)^{1/m}\leq c^{1/m}_{d,n}\left(\sum_{u\in\mathcal{X}}\hat{p}_u^d(1-\hat{p}_u)^d\right)^{1/m}+O\left(n^{1/2-1/2m}\right)
\end{align}
In addition, we have that 
\begin{align}
    \sqrt{\frac{n}{n-1}}=\sqrt{1+\frac{1}{n-1}}\leq 1+\frac{1}{2(n-1)}=1+O\left(\frac{1}{n}\right)
\end{align}
where the inequality follows from Bernoulli's Inequality. 
Putting together the above we obtain
\begin{align}
    &\frac{1}{n}\sqrt{\frac{n}{n-1}}
\left(\frac{1}{\delta_1}\left( \sum_{u\in\mathcal{X}}\sum_{k=1}^{m/2}c_{k,m,n}(\hat{p}_u(1-\hat{p}_u))^k+\epsilon_{n}\right)\right)^{1/m}=\\\nonumber
&\frac{1}{n}\left(\frac{c_{d,m,n}}{\delta_1}\right)^{1/m}\left( \sum_{u\in\mathcal{X}} (\hat{p}_u(1-\hat{p}_u))^{m/2}\right)^{1/m}+O\left(\frac{1}{n^{0.5(1+1/m)}}\right).
\end{align}
Finally, we minimize the coefficient of the empirical term.
First, Theorem $4$ of \cite{skorski2020handy} shows that $c_{d,m,n}\leq (dn)^d$. Hence,  
\begin{align}
    \left(\frac{c_{d,m,n}}{\delta_1}\right)^{1/m}\leq  \left(\frac{1}{\delta_1}\right)^{1/m}(nd)^{1/2}
\end{align}
and
\begin{align}
    m^*=\argmin  \left(\frac{1}{\delta_1}\right)^{1/m}\left(\frac{nm}{2}\right)^{1/2}=\argmin  \frac{1}{m}\log\left(\frac{1}{\delta_1}\right)+\frac{1}{2}\log\left(\frac{nm}{2}\right).
\end{align}
A simple derivation attains $m^*=2\log\frac{1}{\delta_1}$ and
$$\left(\frac{c_{d,m,n}}{\delta_1}\right)^{1/m^*}\leq \sqrt{\exp(1)n\log(1/\delta_1)}$$

\subsection{A proof for Theorem \ref{lower_bound}}\label{lower bound proof}
Let $r$ be the second most probable symbol. That is, $p_{r}=p_{[2]}$. We have that
\begin{align}\label{first}
    \hat{p}_{[1]}-\hat{p}_s\geq \hat{p}_{r}-\hat{p}_{s}\stackrel{\cdot}{\sim}\mathcal{N}(\mu_{rs},\sigma^2_{rs})
\end{align}
where $\mu_{rs}=p_r-p_s$ and $\sigma_{rs}=\sqrt{(p_r(1-p_r)+p_s(1-p_s))-2p_rp_s)/n}$. Hence, 
\begin{align}\label{CDF}
    \mathbb{P}(\hat{p}_{r}-\hat{p}_{s}\leq T_{rs})=\phi\left(\frac{T_{rs}-\mu_{rs}}{\sigma_{rs}}\right)
\end{align}
where $\phi$ is the standard normal cumulative distribution function. Setting (\ref{CDF}) to equal $1-\delta$ yields $T^*_{rs}=\mu_{rs}+z_\delta\sigma_{rs}$.
In other words, 
\begin{align}\label{exact}
    \mathbb{P}(\hat{p}_{r}-\hat{p}_{s}\leq T^*_{rs})=1-\delta.
\end{align}

Now, assume that we found $T$ that satisfies $\mathbb{P}( \hat{p}_{[1]}-\hat{p}_{s}\leq T)\geq 1-\delta$ and $T<T_{rs}^*$. This implies
$$ \mathbb{P}(\hat{p}_{r}-\hat{p}_{s}\leq T)\geq \mathbb{P}(\hat{p}_{[1]}-\hat{p}_{s}\leq T)\geq1-\delta$$
where the first inequality follows from (\ref{first}). This contradicts (\ref{exact}), as desired. Finally, setting $p_r - p_s=O(1/n)$ attains the desired result.

\subsection{A Proof for  Proposition \ref{add-on} }\label{labeled_appendix}
\begin{proposition} \label{add-on}
    Let $p$ be a probability distribution over $\mathcal{X}$. Then, 
    \begin{align}\label{A22}
        p_{[1]}(1-p_{[1]})=\max_{u\in\mathcal{X}}p_u(1-p_u)
    \end{align} where $p_{[1]}=\max_{u\in\mathcal{X}}p_u$ is the largest element in $p$. 
\end{proposition}
\begin{proof}
Let us first consider the case where $p_u\leq 1/2$ for all $u \in \mathcal{X}$. Then (\ref{A22}) follows directly from the montonicity of $p_u(1-p_u)$ for $p_u \in [0,1/2]$.  
Next, assume there exists a single $p_j>1/2$. Specifically, $p_j=1/2+a$ for some positive $a$. Then, the remaining $p_u$'s are necessarily smaller than $1/2$. Further, the supremum of $p_u(1-p_u)$ over $u\neq j$ is attainable at $p_u=1/2-a$, from the same monotonicity reason. This means that $\max_{u \neq j }p_u(1-p_u)\leq(1/2-a)(1-(1/2-a))=(1/2+a)(1-(1/2+a))$ where the second equality follows from the symmetry of $p_u(1-p_u)$ around $p_u=1/2$. 
\end{proof}

\bibliographystyle{plainnat}
\bibliography{bibi}  
\end{document}